\documentclass[sigconf,natbib=true,anonymous=false]{acmart}
\usepackage{algorithm}
\usepackage{algorithmic}
\usepackage{enumerate}
\usepackage{amsmath,amsthm}
\usepackage{url}

\newtheorem{definition}{Definition}%
\newtheorem{lemma}{Lemma}
\newtheorem{theorem}{Theorem}
\usepackage{graphicx}    % 插入图片  
\usepackage{subfig}  % 支持子图标题 
\DeclareSubrefFormat{myparens}{#1(#2)}
\usepackage{makecell}
\usepackage{siunitx}
\usepackage{array}
\usepackage{pifont}

\AtBeginDocument{%
  }

\begin{document}

%%
%% The "title" command has an optional parameter,
%% allowing the author to define a "short title" to be used in page headers.

\title{Data Assetization via Resources-decoupled Federated Learning}

%%
%% The "author" command and its associated commands are used to define
%% the authors and their affiliations.
%% Of note is the shared affiliation of the first two authors, and the
%% "authornote" and "authornotemark" commands
%% used to denote shared contribution to the research.

\author{Jianzhe Zhao}
\affiliation{%
 \institution{Software college, Northeastern University}
 \city{Shenyang}
 \state{Liaoning}
 \country{China}}
 \email{zhaojz@swc.neu.edu.cn}

\author{Feida Zhu}
\affiliation{%
  \institution{School of Computing \& Information Systems, Singapore Management University}
  \city{Singapore}
  \country{Singapore}}
  \email{fdzhu@smu.edu.sg}

\author{Lingyan He}
\affiliation{%
  \institution{Software college, Northeastern University}
   \city{Shenyang}
 \state{Liaoning}
 \country{China}}
\email{2471425@stu.neu.edu.cn}

\author{Zixin Tang}
\affiliation{%
  \institution{Software college, Northeastern University}
   \city{Shenyang}
 \state{Liaoning}
 \country{China}}
\email{20227024@stu.neu.edu.cn}

\author{Mingce Gao}
\affiliation{%
  \institution{Software college, Northeastern University}
   \city{Shenyang}
 \state{Liaoning}
 \country{China}}
\email{20226811@stu.neu.edu.cn}

\author{Shiyu Yang}
\affiliation{%
  \institution{Software college, Northeastern University}
   \city{Shenyang}
 \state{Liaoning}
 \country{China}}
\email{20216732@stu.neu.edu.cn}

\author{Guibing Guo}
\affiliation{%
  \institution{Software college, Northeastern University}
   \city{Shenyang}
 \state{Liaoning}
 \country{China}}
\email{guogb@swc.neu.edu.cn}

%%
%% By default, the full list of authors will be used in the page
%% headers. Often, this list is too long, and will overlap
%% other information printed in the page headers. This command allows
%% the author to define a more concise list
%% of authors' names for this purpose.
\renewcommand{\shortauthors}{Trovato et al.}

%%
%% The abstract is a short summary of the work to be presented in the
%% article.
\begin{abstract}
With the development of the digital economy, data is increasingly recognized as an essential resource for both work and life. However, due to privacy concerns, data owners tend to maximize the value of data through the circulation of information rather than direct data transfer. Federated learning (FL) provides an effective approach to collaborative training models while preserving privacy. However, as model parameters and training data grow, there are not only real differences in data resources between different data owners, but also mismatches between data and computing resources. These challenges lead to inadequate collaboration among data owners, compute centers, and model owners, reducing the global utility of the three parties and the effectiveness of data assetization. In this work, we first propose a framework for resource-decoupled FL involving three parties. Then, we design a Tripartite Stackelberg Model and theoretically analyze the Stackelberg-Nash equilibrium (SNE) for participants to optimize global utility. Next, we propose the Quality-aware Dynamic Resources-decoupled FL algorithm (QD-RDFL), in which we derive and solve the optimal strategies of all parties to achieve SNE using backward induction. We also design a dynamic optimization mechanism to improve the optimal strategy profile by evaluating the contribution of data quality from data owners to the global model during real training. Finally, our extensive experiments demonstrate that our method effectively encourages the linkage of the three parties involved, maximizing the global utility and value of data assets.

\end{abstract}

%%
%% The code below is generated by the tool at http://dl.acm.org/ccs.cfm.
%% Please copy and paste the code instead of the example below.
%%



%%
%% Keywords. The author(s) should pick words that accurately describe
%% the work being presented. Separate the keywords with commas.
%% A "teaser" image appears between the author and affiliation
%% information and the body of the document, and typically spans the
%% page.
% \received{20 February 2007}
% \received[revised]{12 March 2009}
% \received[accepted]{5 June 2009}

%%
%% This command processes the author and affiliation and title
%% information and builds the first part of the formatted document.
\maketitle

\section{Introduction}
With the rapid development of the digital economy, data have unlocked the potential of machine learning (ML), driving the growth of various industries and becoming a key factor of production in the new era \cite{bib-1}. Data assets have measurable value, can be quantified in monetary terms, and generate economic benefits for organizations, such as improving productivity and optimizing business processes \cite{bib-2}. However, transforming data into assets through market transactions is clumsy. Although there are many ingenious data pricing models \cite{bib-3,bib-4,bib-5}, direct data market transactions are a cumbersome way of data assetization. First, for privacy protection reasons, data owners (organizations or institutions) may not be willing to participate in data transactions and are subject to legal and regulatory restrictions, such as the General Data Protection Regulation (GDPR) \cite{bib-6}, which requires that data remain where it is generated. Secondly, data owners are more concerned about how the information (knowledge) embedded in the data can improve their business efficiency \cite{bib-7}. Therefore, they are more inclined to enable the circulation of information to maximize the value of the data rather than direct data transactions. In recent years, federated learning (FL) has emerged as a successful paradigm for collaborative model training under privacy constraints \cite{bib-8,bib-9}. It is widely adopted and benefits the finance, insurance, and healthcare industries \cite{bib-10,bib-41,bib-42,bib-43}. FL keeps data locally while training a global model collaboratively, thus offering an effective way for data assetization. %Avoiding direct data transfer, FL offers an effective way to data assetization through the circulation of information. %FL maintains local data while training a global model collaboratively, avoiding direct data transfer. It offers an effective way to data assetization through the circulation of information.

However, to maximize the potential value of data assets by FL training high-quality global models, the inherent challenge is heterogeneous client data resources, i.e., data owners exhibit differences in data quantity and quality \cite{bib-11,bib-12,bib-13}. Since quantity and quality of data are key factors in determining the effectiveness of ML models, the data heterogeneity can cause high-quality data owners to lose the motivation to participate in FL. Moreover, a more pressing challenge is the mismatch between data and computing resources. With recent advances in artificial intelligence, the positive correlation between model performance and the number of model parameters has been well-established \cite{bib-14,bib-15}. Large language models have recently garnered significant attention in academia and industry because of their outstanding performance. However, as the number of parameters increases, data owners face the dilemma of insufficient computing resources. We conducted preliminary experiments for different computing resource configurations, model sizes (3B, 7B, and 13B), and training durations required with a fixed amount of data (140B tokens) \cite{bib-16}. The results in Tab. \ref{tab1} indicate that within acceptable training times, the lower limit of computing resources exceeds most data owners' typical computing capabilities. We refer to this real-world scenario as resource-decoupled FL, where data resources with unbalanced quality and quantity are distributed across different data owners and a mismatch between data and computing resources, as shown in Fig. \ref{fig1}.

\begin{figure}[tb]
  \centering
  \includegraphics[width=\linewidth]{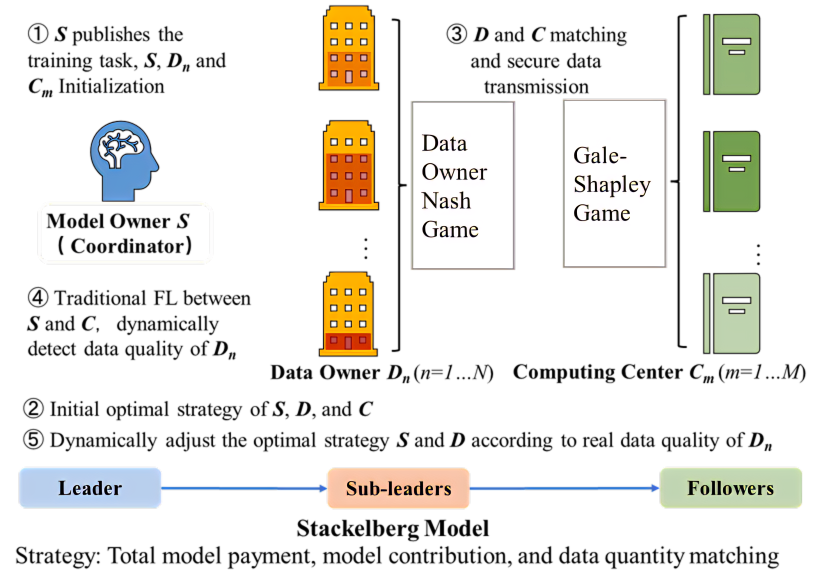}
  \vspace{-20pt}
  \caption{Data Assetaization via Resource-Decoupled FL}
  \Description{A scenario of resource-decoupled resources in FL, where data resources with unbalanced quality and quantity are distributed across different data owners and a mismatch between data resources and computing resources.}\label{fig1}
\end{figure}

\begin{table}[htbp]
\renewcommand{\arraystretch}{0.9} % 调整行距
\centering
\caption{GPU Configurations, Model Size, and Training Performance}
\vspace{-10pt}
\begin{tabular}{lcc}
\toprule
\textbf{GPU * Quantity} & \textbf{\makecell{Model Size(B)}} & \textbf{\makecell{Training Time(days)}}\\ \midrule
H100 (80GB) * 1  & 7  & 133.28\\ 
H100 (80GB) * 1  & 13 & 274.51\\
H100 (80GB) * 1  & 30 & 571.18\\ 
H100 (80GB) * 2  & 65 &  618.78\\ \hline
A100 (40GB) * 1  & 7  & 646.30 \\
A100 (40GB) * 1  & 13 & 1200.27\\ 
A100 (40GB) * 1  & 30 & 2769.86\\
A100 (40GB) * 2  & 65 & 3000.69\\ \hline
V100S (32GB) * 1 & 7  & 1551.12\\
V100S (32GB) * 1 & 13 & 2880.66\\ 
V100S (32GB) * 2 & 30 & 3323.84\\
V100S (32GB) * 3 & 65 & 4801.10\\  
\bottomrule
\end{tabular}

\label{tab1}
\vspace{-5pt}
\end{table}

Focusing on the problem of heterogeneous data resources in FL, researchers have developed advanced incentive mechanisms that employ game theory, contracts, and auctions for interactions between the cloud server and multiple clients \cite{bib-17,bib-22,bib-23,bib-24, bib-25, bib-26}. Among them, game theory formulates the interaction between incentive structures and uses mathematical theories to solve individual optimization strategies in complex environments with competitive phenomena (such as FL). Stackelberg games provide solutions for non-cooperative games and are applied to FL scenarios with multiple roles \cite{bib-7,bib-11,bib-17}. There are also many dynamic game solutions for complex decision processes \cite{bib-27,bib-28}. Many researchers measure the contribution of client data and provide a fair solution for profit distribution, encouraging data owners to participate in FL and improving the performance of the global model. Some metrics are used to evaluate the contribution, such as verification accuracy and loss values \cite{bib-17, bib-18, bib-13}. Moreover, the Shapley values successfully quantify marginal contributions \cite{bib-19, bib-20, bib-17} and provide a fair incentive solution for cooperative games. However, the classical Shapley value calculation is computationally expensive and impractical in FL. Therefore, efficient approximations Shapley calculation have been proposed \cite{bib-21}. Although these game-theoretic incentive mechanisms have achieved results in server-client FL, they cannot be directly applied to resource-decoupled FL scenarios. Due to the separation of computing and data resources, the game should be complicated by introducing new roles. 

Unlike previous server-client models, we identify three roles in the resource-decoupled FL: model owner (coordinator), data owners, and computing centers. In Fig. \ref{fig1}, the model owner initiates FL training to build an optimal global model that maximizes the overall utility. Data owners hold data of varying quality and quantity. Their objective is to transform their data into advanced knowledge to create business value, such as models and monetary benefits, ultimately maximizing personal data assets. Computing centers are large-scale data centers with abundant computing resources that seek to increase their rental income. For a specific business scenario, the model owner coordinates $N$ data owners and $M$ computing centers in FL. In this work, we are dedicated to maximizing the value of the data assets. To incentivize tripartite collaboration, a reward scheme should be designed. This scheme encourages data owners to contribute high-quality and high-quantity data while computing centers align their computing resources with the data volume to achieve maximum revenue. Therefore, it maximizes the profit of personal data, the rental income of computing resources, and the utility of the global model. In light of the analyses above, when the utility of the global model reaches its maximum, all parties can achieve higher profits, ensuring optimal global utility. Therefore, the new challenges include: (C1) How can FL be effectively implemented in scenarios where data owners and computing centers are separated? (C2) How can an incentive mechanism be designed that encourages all parties to adopt optimal strategies, thus achieving global utility optimization? (C3) How do we assess the contribution of data owners to the final model and quantify their rewards?

Motivated by this, we focus on the approach of data assetization, that data do not circulate but the information does. We first propose a resource-decoupled FL framework for solving (C1), in which the model owner coordinates the FL training of the three parties, the data centers are in charge of providing the data, and the matching data centers are responsible for local training. Then, we design a Tripartite Stackelberg Model to encourage three parties to optimize their profit, thus solving (C2), in which the model owner is the leader, the data owners are sub-leaders, and the computing centers are followers. We derive three utility functions according to each one's profit and analyze the Nash equilibrium of the sub-game models. Global utility optimization is defined as the achievement of the Stackelberg-Nash equilibrium (SNE), and backward induction is introduced to solve the model, resulting in three-party strategies that maximize global utility. Finally, to overcome (C3), we evaluate the model contribution of data owners according to data quality, and an adjustment mechanism is designed to improve the optimal strategy, encouraging data owners to provide high-quality and quantity data in FL training to optimize the global model's performance.

Specifically, our contributions are as follows.
\begin{itemize}
 \item We propose a resource-decoupled FL framework that involves model owners, data owners, and computing centers and define utility functions and the global utility optimization problem to maximize the value of data assets.
\item We develop a Tripartite Stackelberg model and analyze the SNE of the model. Using backward induction, we derive optimal strategies to maximize global utility.
\item We introduce the Quality-aware dynamic resource-decoupled FL algorithm (QD-RDFL), in which a dynamic optimization mechanism is designed to improve optimal strategies by evaluating the contribution of data owners.
\item The comprehensive experiments demonstrate that our method effectively encourages participants, maximizing global utility and the value of data assets.
\end{itemize}

%The rest of the paper is organized as follows. Section \ref{sec_2} reviews the related works. Section \ref{sec_3} presents the problem statement and the specific utility functions, introduces and analyzes the Tripartite Stackelberg Model, and concludes with the proposal of the QD-RDFL. Section \ref{sec_4} evaluates the performance of our QD-RDFL, and Section \ref{sec_5} concludes this work and discusses future work.

\section{Related Works}\label{sec_2}

\textbf{Data Valuation in ML: }
To encourage participants to contribute higher quality data and actively participate in FL, accurately assessing the contributions of participants is critical for an equitable distribution of benefits. Existing research primarily evaluates contributions based on the performance of the final model \cite{bib-22,bib-23,bib-24}. Some studies rely on previous data provided by clients or subjectively define data quality \cite{bib-11,bib-20,bib-25,bib-26}. In contrast, others assess contributions through actual performance metrics during training, such as validation accuracy, loss values, and Shapley values \cite{bib-17,bib-18,bib-19}. To objectively quantify the contributions of data owners in FL and design reward mechanisms that encourage the provision of high-quality data, we evaluate the data value based on actual training performance, considering data quantity and quality.

\textbf{Incentive Mechanisms in FL:}
Game-theoretical incentive mechanisms simulate strategic interactions between the server and the clients and the competitive dynamics among clients. As a result, these mechanisms have garnered extensive attention in FL. The classical Shapley value provides a robust framework to assess each participant's contribution, which can be leveraged to design incentive mechanisms for cooperative games \cite{bib-12,bib-21}. However, traditional Shapley values are computationally intensive, increasing the workload of FL training. Incentive mechanisms based on Stackelberg games position the server as the leader and clients as followers, typically maximizing the utility for both parties. For example, \cite{bib-17} designed a three-layer Stackelberg game for FL involving cloud-edge servers and clients, optimizing social utility through Pareto utility enhancements. \cite{bib-23,bib-25} used a two-stage Stackelberg game-based incentive mechanism, where the server first engages in a Stackelberg game to maximize its utility, followed by clients participating in a non-cooperative game to achieve their utility maximization through Nash equilibrium solutions. Additionally, \cite{bib-27} introduced a dynamic game solution for complex decision-making processes between servers and devices in the FL process as a continuous iterative game to optimize social welfare. \cite{bib-28} constructed a two-layer game architecture, where the upper layer mechanism dynamically adjusts selection strategies through reinforcement learning to achieve long-term utility maximization, and the lower layer mechanism benefits individuals through evolutionary games.

However, these incentive mechanisms are explicitly designed for client-server architectures and cannot be directly applied to the resource-decoupled FL scenario. In this context, it is essential to formulate and refine utility functions for the three parties involved, as well as to develop novel incentive mechanisms to achieve the goal of maximizing the value of data assets.

\section{Methodology}\label{sec_3}

\subsection{Problem Statement}

This work focuses on maximizing the value of data assets through FL in a resource-decoupled scenario. The ultimate profit of the three parties is linked to the quality of the global model. However, predicting the model's quality is challenging. Therefore, we define the problem as a global utility optimization in resource-decoupled FL.%, transforming global quality optimization into utility optimization. 
In this section, we first formalize the global optimization problem and then provide the utility functions of three parties. The necessary notations are listed in Tab. \ref{tab2}.

\begin{table}[htb]
\centering
\caption{Notations and Descriptions}
\vspace{-10pt}
\begin{tabular}{c l}
\toprule
\textbf{Notations} & \makecell[l]{\textbf{Descriptions}}\\ \midrule
$S$, $D$, $C$  & \makecell[l]{Model owner (coordinator/server), \\ data owner, computing center}\\
$U_s$, $U_n$, $U_m$  & \makecell[l]{Utility of $S$, $D_n$, $C_m$}\\
$N$, $M$           & \makecell[l]{Number of $D$ and $C$} \\
$\lambda$          & \makecell[l]{Market regulating factor} \\
$d_m$ , $\left|d_m\right|$ & \makecell[l]{Data quantity and max quantity undertaken by $C_m$} \\
$\rho$             & \makecell[l]{Unit training payment of $D$} \\
$x_n$              & \makecell[l]{Data quantity provided by $D_n$} \\
$q_n$              & \makecell[l]{Model quality contribution of $D_n$} \\
$f_n$              & \makecell[l]{Data quality of training by $D_n$} \\
$\varepsilon$      & \makecell[l]{Unit training cost of $C$} \\
$\sigma_m$         & \makecell[l]{Computational power of $C_m$} \\
$\eta$             & \makecell[l]{Total model payment from the model owner} \\
$\left|X_n\right|$ & \makecell[l]{Maximum data volume of $D_n$} \\
$\xi$              & \makecell[l]{Data quality threshold of training} \\
$\alpha$           & \makecell[l]{Adjustment factor of model quality} \\
$G$           & \makecell[l]{Matching for data owner and computing center} \\
$f_m$           & \makecell[l]{Data quality of $C_m$ calculate during training} \\
$w$           & \makecell[l]{Model parameter} \\
$\beta$           & \makecell[l]{Learning rate for local model parameter update} \\
$loss(t)$          & \makecell[l]{Average loss at specific time $t$ in a iteration } \\
$L$          & \makecell[l]{Optimal adjustment round} \\
\bottomrule
\end{tabular}
\label{tab2}
\end{table}
\textbf{Global Utility Optimization:} We will design an incentive mechanism to derive the optimization strategy for the three parties to optimize global utility $U_{Global}$. The sub-optimization problems include the model owner's utility $U_s$, the utility $U_n$ of data owner $D_n$, and the utility $U_m$ of computing center $C_m$. As in the previous analysis, the overall utility $\sum_NU_n$ of data owners and the overall utility $\sum_MU_m$ of the computing centers increase when an optimum $U_s$ is obtained. Since $U_s$ is up to the quality of the global model, we need to evaluate the contribution of quantity and quality of data provided by data owners to global model updates. These evaluations encourage data owners to provide higher-quality model updates, thus improving the global model.
However, due to the separation of computing and data resources, FL tasks require data owners to be matched with computing centers. Consequently, computing centers maximize their utility by optimally matching with the data owners.

In resource-decoupled FL, the computing center acts as the 'client' in classical FL, such as FedAvg \cite{bib-9}, after matching with the specific data owner. Therefore, the resource-decoupled FL aims to train an optimal global model $w^*$, minimizing the global loss function $F(w)$ by aggregating the high-quality local update $F_i(w)$, as shown in Equ. \ref{equ-1}.
\begin{equation}
\label{equ-1}
\mathop{\min}_{w^*} F(w):=\sum_{i=1}^{K}p_i F_i(w).
\end{equation}
where $K$ represents the number of computing centers matching successfully. We assume $K\leq N \leq M$ in this work. $p_i=\frac{x_n}{\sum_Kx_n}$ is data provided ratio of data owner $D_n$.

To encourage the three parties, we analyze and formalize the utility of each computing center, data owner, and model owner as the basis of the reward distribution. We evaluate the model contribution of the data owner based on the data quality and quantity and transform the data value into model value. Meanwhile, considering both cost and profits, we derive the utility functions of the three parties.

\textbf{Utility of Computing Center:} The utility of computing center $U_m$ is the difference between the cost and profit gained from renting out computing power. Since the profit is related to the matching data quantity provided by the data owner, the larger the quantity, the higher the profit. However, the cost of completing the same task varies due to differences in hardware devices and computing power across computing centers\cite{bib-34,bib-35,bib-36}. The goal of $C_m$ is to determine the optimal matching with a specific data owner according to undertaking data quantity $d_m$ to maximize its utility $U_m$. Accordingly, we define $U_m$ of $C_m$ as Equ. \ref{equ-2}.
\begin{flalign}\label{equ-2}
&\
\begin{array}{l} 
  \left\{
  \begin{aligned}
    &\mathop{\arg\max}\limits_{d_{m}} \quad U_{m}(d_{m}) = \lambda \frac{d_{m}}{\sum_{m=1}^{M} d_{m}} \sum_{n=1}^{N} \rho x_{n} - \varepsilon \sigma_{m} d_{m}, \\[5pt]
    &s.t. \quad \rho \geq 0.
  \end{aligned}
  \right.
\end{array}&
\end{flalign}
where $\sigma_m$ is the computational power factor of $C_m$, which reflects the difference in cost across various computing centers. $\epsilon$ is the unit training cost of the computing center. $\rho$ is the unit training payment of the data owner. $\lambda$ is the market regulating factor to adjust the two-sided marketplace between data owners and computing centers. In this work, we assume $M=N$; at this point, $\lambda$ is 1.

%When $M<N$, $\lambda$ is a constant greater than 1; when $M>N$, $\lambda$ is less than 1; and when $M=N$, $\lambda$ is 1, which is the case discussed in this work.

\textbf{Utility of Data Owner:} The utility of data owner $U_n$ is the difference between the model profit generated by their data and the cost of model training. Specifically, the model profit of personal data originates from the total model payment provided by model owners, and the cost arises from the fee for renting a computing center per training unit. Following \cite{bib-13,bib-17}, we evaluate the contribution $q_n$ of each data owner to the global model based on the quantity $x_n$ and quality $f_n$ of their training data. Thus, given a certain total model payment $\eta$, the higher the contribution, the higher the reward data owners receive. The evaluations encourage data owners to provide higher quality and larger quantity of data in FL. The $U_n$ of $D_n$ is defined as shown in Equ. \ref{equ-3}.
\begin{flalign}\label{equ-3}
&\
\begin{array}{l} 
  \left\{
  \begin{aligned}
    &\mathop{\arg\max}\limits_{q_{n}} \quad U_{n}(q_{n},q_{-n}) = \frac{q_{n}}{\sum_{n=1}^{N}q_{n}} \eta - \lambda \rho x_{n}, \\[5pt]
    &s.t. \quad 
    \left\{
    \begin{aligned}
        &q_{n} = f_{n} x_{n}, \\
        &0 < x_{n} < \left| X_{n} \right|, \\
        &f_{n} \geq \xi. \\ 
    \end{aligned}
    \right.
  \end{aligned}
  \right.    
\end{array} &
\end{flalign}
where $| \cdot |$ denotes the number of elements in a set, and $q_n$ represents the model quality contribution of $D_n$. Specifically, $q_n$ is determined by the product of data quality and data quantity. $q_{-n}=\{q_1, q_2, ..., q_{n-1}, q_{n+1},....,q_N\}$  denotes the collective model quality contribution of overall data owners except $D_n$.

\textbf{Utility of Model Owner:} The utility of the model owner $U_s$ is the difference between the global model profit and the total model payment $\eta$ to data owners. As known, the higher the model contributions from the data owners, the higher quality of the global model, i.e., global model profit. Therefore, we define the quality of the global model as a concave function $\alpha\cdot g(\sum^{N}_{n=1}q_n)$ which depends on the model contributions of all data owners, where $\alpha > 0$ is a preset parameter. The $U_s$ of model owner $S$ is defined as Equ. \ref{equ-4}.
\begin{flalign}\label{equ-4}
&\
\begin{array}{l} 
  \left\{\begin{matrix} 
  \mathop{\arg\max}\limits_{\eta}\quad U_{s}(\eta) \quad n\in \left [1,N \right ],  \\
  s.t.\quad  \left\{\begin{array}{lc}
    U_{s}(\eta)=\alpha \cdot g(\sum_{n=1}^{N}q_{n})-\eta, \\
    \eta\ge 0.
  \end{array}\right.
\end{matrix}\right.    
\end{array} &
\end{flalign}

%Based on the above analysis, the problem is to encourage the model owner, data owners, and computing centers in resources-decoupled FL to jointly generate a strategy profile to achieve optimal global utility at equilibrium.

\subsection{ Tripartite Stackelberg Model}
In this section, we specify the strategies of $S$, $D$, and $C$, each trying to maximize its profit. We model the Stackelberg game for three parties and define the SNE to achieve global utility optimization.

The strategy profile denoted as $\left\langle \eta^*, Q^*, G^* \right\rangle$, consists of the strategies of the model owner $S$, the data owners $D_n ( n = 1, \ldots, N )$, and the computing centers $C_m  ( m = 1, \ldots, M )$. These strategies are taken in order: first, the model owner $S$ determines the optimal total model payment $\eta^*$; then, each data owner $D_n$ decides the quantity and quality of data to provide, i.e., the model quality contribution $q_n^*$; finally, each computing center $C_m$ determines the beneficial matching relationship $G^*_m$ with the specific data owner based on the optimal quantity of data undertaking $d^*_m$. To protect data privacy, we perform a one-time matching between data owners and the computing centers. In this manner, we formalize the resource-decoupled FL strategy as a Tripartite Stackelberg Model in Def. \ref{def-1}, where the model owner acts as the leader, the data owners act as secondary leaders, and the computing centers act as followers. The Stackelberg model establishes the sequential achievement of sub-utility optimization goals based on the major-minor relationships between the three parties. Within this framework, $N$ data owners and $M$ computing centers adopt strategies in their inner Nash game, ultimately achieving global utility optimization.

\begin{definition}[Tripartite Stackelberg Model]\label{def-1}
The game comprises the strategies of the model owner $S $, the data owners $D_n( n = 1, \ldots, N )$, and the computing centers $C_m( m = 1, \ldots, M )$ to achieve global utility optimization in resource-decoupled FL.
\begin{equation}\label{equ-5}
\left\{ \begin{array}{ll}
    \text{Leader } S: \quad \eta^* = \arg\max\limits_{\eta} U_s, \\[5pt]
    \text{Sub-Leaders } D_n: \quad q_n^* = \arg\max\limits_{q_n} U_n, & n = 1, \dots, N, \\[5pt]
    \text{Followers } C_m: \quad G^*_m = \arg\max\limits_{G_m} U_m, & m = 1, \dots, M.
\end{array} \right.
\end{equation}
\end{definition}

We aim to generate the strategy profile $\left\langle \eta^*, Q^*, G^* \right\rangle$ of the tripartite Stackelberg model defined by Equ. \ref{equ-5} to achieve the SNE.
Specifically, for the sub-game between data owners, we derive its corresponding Nash equilibrium \textbf{(Appendix \ref{App-1})}, and the decisions of the three parties collectively achieve the SNE. We define SNE in the resource-decoupled FL, as shown in Equ. \ref{equ-6}.

\begin{definition}[ Stackelberg-Nash equilibrium]\label{def-2}
An optimal strategy profile <$\eta^*, Q^*, G^*$> constitutes a Stackelberg-Nash equilibrium if and only if the following set of inequalities is satisfied.
\end{definition}

\begin{equation}\label{equ-6}
\left\{ \begin{array}{ll}
U_s(\eta^*,Q^*, G^* ) \geq U_s(\eta, Q^*, G^*), \\[8pt]
U_n(\eta^*, q_n^*, q_{-n}^*, G^*) \geq U_n(\eta^*, q_n, q_{-n}^*, G^*), & n = 1, \ldots, N, \\[8pt]
U_m(\eta^*, Q^*, G_m^*) \geq U_m(\eta^*, Q^*, G_m), & m = 1,\ldots, M.
\end{array} \right.
\end{equation}
where $ Q^*=\{q^*_1, q^*_2, ..., q^*_N\}$ is the Nash equilibrium solutions of sub-games. SNE indicates that each of the three parties takes its optimal strategy, which maximizes its utility. No one can add its own utility by unilaterally changing its strategy.

\subsection{Analysis of Tripartite Stackelberg Model}\label{sec-3.3}
In this section, we prove the existence and uniqueness of the SNE strategy and explore the methods to identify it. The strategies of the followers can be described as effectively responding to the optimal strategy provided by the leader. Based on backward induction \cite{bib-29}, we first derive the optimal $d_m^*$ for each computing center, followed by the optimal $q_n^*$ for each data owner, and finally the optimal $\eta^*$ for the model owner. The optimal strategy combinations of the three parties that achieve SNE. 

\begin{lemma}\label{lemma-1}
For computing center $C_m$, there is a unique optimal $d_m^*$ to maximize its utility $U_m$ for the given data quantity from the data owners.
\end{lemma}
\begin{proof}
The policy space of the computing center is the quantity of data undertaken. This interval is $[0, |d_m|]$, which is bounded, closed, and continuous. Based on the utility function $U_m$, we obtain the first and second derivatives of $U_m$.

\begin{equation}\label{equ-30}
\frac{\partial U_n(d_m)}{\partial d_m}=\frac{\sum_{i\neq m} d_i}{(d_m+\sum_{i\neq m} d_i)^2}\lambda \sum_{n=1}^{N} \rho x_{n} - \varepsilon \sigma_{m}
\end{equation}
and
\begin{equation}\label{equ-31}
\frac{\partial ^2 U_n(d_m)}{\partial ^2 d_m}=\frac{-2 \sum_{i\neq m} d_i}{(d_m+\sum_{i\neq m} d_i)^3}\lambda \sum_{n=1}^{N} \rho x_{n} < 0
\end{equation}
\end{proof}
% Subsequently, we set the first derivative in Equ.\ref{equ-30} to zero to obtain the optimal strategy for the computing center.
For every computing center there exists a $d_m^*$ such that $U_m$ is optimal, we can also prove that there exists a unique optimal Nash equilibrium strategy, and the proof is similar to Lemma \ref{lemma-2}. Accordingly, the derived $d_m^*$ is given in Equ. \ref{equ-32}.
\begin{equation}\label{equ-32}
d_m^*=\lambda\frac{(M-1)\sum_{n=1}^{N} \rho x_{n}}{\varepsilon\sum_{i=1}^M{\sigma_{i}}}(1-\sigma_m\frac{M-1}{\sum_{i=1}^M{\sigma_{i}}})
\end{equation}

\begin{lemma}\label{lemma-2}
For data owner $q_n$, there is a unique Nash equilibrium optimal strategy $q_n^*$ to maximize its $U_n$ for the given $\eta$ and $q_{-n}$.
\end{lemma}
According to \cite{bib-7}, a unique optimal Nash equilibrium solution exists if and only if the following two conditions are satisfied: (1) the set of strategies is convex, bounded, and closed; (2) the utility function in the strategy space is quasi-concave and continuous.

\begin{proof}
The policy space for the data owner is to decide how much model quality to contribute. This interval is $[0, |Xn|]$, which is bounded, closed, and continuous. According to the utility function $U_n$ of the data owner, we obtain the first and second-order derivatives of $U_n$ as shown in Equ. \ref{equ-8} and Equ. \ref{equ-9}.

\begin{equation}\label{equ-8}
\frac{\partial U_n(q_n,q_{-n})}{\partial q_n}=\frac{\sum_{i\neq n} q_i}{(q_n+\sum_{i\neq n} q_i)^2}\eta - \frac{\lambda\rho}{f_n}
\end{equation}
and
\begin{equation}\label{equ-9}
\frac{\partial ^2 U_n(q_n,q_{-n})}{\partial ^2 q_n}=\frac{-2 \sum_{i\neq n} q_i}{(q_n+\sum_{i\neq n} q_i)^3}\eta < 0
\end{equation}

The strategy of the data owner $ Q^*=\{q^*_1, q^*_2, ..., q^*_N\}$ that we need to solve is the Nash equilibrium solution, so we derive and proof of Nash equilibrium is detailed in \textbf{Appendix \ref{App-1}}.
\end{proof}

Setting Equ.\ref{equ-9} equal to zero, we can deduce the optimal response strategy for a single data owner as Equ. \ref{equ-11}.

\begin{equation}\label{equ-11}
q_n^*=\frac{(N-1)\eta}{\lambda\rho\sum_{i=1}^N{\frac{1}{f_i}}}(1-\frac{N-1}{f_n\sum_{i=1}^N{\frac{1}{f_i}}})
\end{equation}

\begin{lemma}\label{lemma-3}
For the model owner, there is a unique optimal $\eta^*$ to optimize the utility $U_s$ for given the parameter of $\alpha$ and the objective data quality provided by the data owner $f_n, n\in (0, N]$.
\end{lemma}

\begin{proof}

The range of $\eta$ is $[0, \infty]$ for each model owner. Thus, this set of strategies is clearly convex and continuous. Although the value of $\eta$ can be infinite in theory, practically, the model owner will not give an infinite total model payment, and there is also the effect of diminishing marginal utility. This leads to the fact that $\eta$ is necessarily bounded and closed. We obtain the first-order and second-order derivatives as shown in Equ. \ref{eq:formula 12} and Equ. \ref{eq:formula 13}.

\begin{equation}\label{eq:formula 12}
\frac{\partial U_s(\eta)}{\partial \eta}=\alpha g^{'}(\sum_{n=1}^Nq_n)\sum_{i=1}^NT_n-1
\end{equation}
and
\begin{equation}\label{eq:formula 13}
\frac{\partial^2 U_s(\eta)}{\partial^2 \eta}=\alpha g^{''}(\sum_{n=1}^Nq_n)(\sum_{i=1}^NT_n)^2 < 0
\end{equation}
where $T_n=\frac{(N-1)}{\lambda\rho\sum_{i=1}^N{\frac{1}{f_i}}}(1-\frac{N-1}{f_n\sum_{i=1}^N{\frac{1}{f_i}}})$. Since $g(\cdot)$ is a convex function and to simulate the diminishing marginal benefit, we set $g(x)=ln(1+x)$ in this work. 
\end{proof}

By setting the first derivative equation to zero, we can compute the optimal model total payment for the model owner as follows.
\begin{equation}\label{eq:formula 14}
\eta ^{*}=\alpha-\frac{1}{\sum_{i=1}^NT_n}
\end{equation}

\begin{theorem}
The complete optimal strategy profile $\left\langle \eta^*, Q^*, G^* \right\rangle$ determined by the backward induction approach uniquely constitutes the Stackelberg-Nash equilibrium.
\end{theorem}
\begin{proof}
See \textbf{Appendix \ref{App-2}}.   
\end{proof}

\subsection{QD-RDFL}

In this section, we design the QD-RDFL algorithm to identify the optimal strategy profile, in which a dynamic optimization mechanism is designed to improve the optimal strategy by evaluating the model contribution of data owners. 

Using the backward induction method\cite{bib-39,bib-40}, we first calculate the optimal strategy of the model owner $\eta^*$ according to Equ. \ref{eq:formula 14}. We then calculate the optimal strategy $Q^*_n$ for $D_n$ according to Equ.\ref{equ-9}. Since $Q_n$ is evaluated along two dimensions, data quality is a posterior metric the server assesses after training, and data quantity $X^*_n$ is a decision determined immediately. Therefore, we initialize $D_n$ with an initial $f_n$ for computation and dynamically adjust $f_n$ based on the real training results. Finally, we calculate the optimal $d^*_m$ according to Equ. \ref{equ-32}, and construct the optimal match $G^*$ using the Gale-Shapley algorithm \cite{bib-30,bib-31}. The matching remains unchanged during subsequent computations for private protection, which means adjusting $f_n$ will not affect the matching relationship.

To evaluate the real contribution of data quality, the most accurate method is to evaluate the test accuracy of each local model. However, it incurs additional costs and resource consumption. In contrast, accessing training loss does not require extra overhead, as it is generated during training. Therefore, we evaluate $f_n$ using the training loss during each iteration as shown in Equ. \ref{eq:formula7}.

\begin{equation}\label{eq:formula7}
\text{$f = loss(t_s)- loss(t_e)$}.
\end{equation}
where $\text{loss}(t)$ denotes the loss in a specific round $t$, $t_s$ represents the start time of validation, and $t_e$ represents the end time. Clients must submit their local model updates within the specified time range; otherwise, the edge server will reject the submission.

% 这里加入算法伪代码图
\begin{algorithm}
\caption{QD-RDFL}
\begin{algorithmic}[1]
% \STATE Initialize Data Owner $C$, Computing Center $M$
% \FOR{each $C_i, i \in [1, N]$}
%     \STATE Send $f_i$ to the server
% \ENDFOR
\REQUIRE $D_n$ ($n = 1, \ldots, N$), and $C_m$ ($m = 1, \ldots, M$), $\rho$, $\varepsilon$, $\alpha$, $\lambda$, and adjusting round $L$
\ENSURE the optimal strategy profile <$\eta^*, Q^*, G^*$> for model owner, data owners, and computing centers
\STATE Initialization of $S$, $D$, $C$, report initial $f_n$
\STATE Calculate $\eta^{*}$ // according to Equ. \ref{eq:formula 14}
\STATE Calculate $x_n^*$ for $D_n, n\in N$ // according to Equ. \ref{equ-11}
\STATE Calculate $d_m^*$ for $C_m, m\in M$ // according to Equ. \ref{equ-32}
\STATE $G^* = Gale-Shapley(D_n, C_m)$ // call Gale-Shapley 
\FOR{each $D_n, n \in [1, N]$}
    \STATE Randomly pick $x_n^*$ and secure data transfer according to $G^*$
\ENDFOR
% \STATE $x^{1*}_N = x^*_i$
\FOR{each round $t=1,2,..., T$}
      \FOR{each $C_m, m\in [1,M]$}
            \STATE $w^{t+1}_m, f^{t+1}_m\leftarrow \text{ClientUpdate}(C_m, w^t)$
        \ENDFOR
        \IF{$t = L$}
        \STATE Normalized each $f_m$ and recalculate $\eta^*$ and $x^*_{n}$
        \IF{$x^*_{t+1} > x^*_t$}
            \STATE Pay additional training fees
        
        \ENDIF
        \ENDIF
        \STATE $w^{t+1}\leftarrow \frac{1}{M}\sum^M_{m=1} w^{t+1}_m$ // global model aggregation
\ENDFOR
\end{algorithmic}
\end{algorithm}

%\vspace{-10pt} 

\begin{algorithm}
\caption{ClientUpdate}
\begin{algorithmic}[1]
\REQUIRE computing center $C_m$, global model parameter $w$, start time $t_s$, and end time $t_e$
\ENSURE local model parameter $w_m$, data quality contribution $f_m$
\STATE $w_m=w$
\FOR{each local epoch $i$ from 1 to $E$}
    \STATE $w_m \leftarrow w_m - \beta \nabla l(w_m)$
    \STATE record $loss(t_s)$ and $loss(t_e)$
    % \IF{epoch = 1}
    %     \STATE record loss $t_s$
    % \ENDIF
    % \IF{epoch = E}
    %     \STATE record $t_e$
    % \ENDIF
\ENDFOR
\STATE $f_m = loss(t_s)- loss(t_e)$
\RETURN $w_m$, $f_m$
\end{algorithmic}
\end{algorithm}

Consequently, the steps of QD-RDFL are outlined as follows, and the details are presented in Alg. 1 and Alg. 2.

\begin{itemize}
    \item[\ding{202}] \textbf{Step 1: Initialization.} $S$ distributes the training tasks, and $S$, $D_n$ ($n = 1, \ldots, N$), and $C_m$ ($m = 1, \ldots, M$) are initialized.
    \item[\ding{203}] \textbf{Step 2: Initial Optimal Strategy.} $S$, $D$, and $C$ compute the optimal strategy profile $\langle \eta^*, Q^*, d^* \rangle$ for the first time based on Equ. \ref{equ-5} and Equ.\ref{equ-6}.
    \item[\ding{204}] \textbf{Step 3: Gale-Shapley Matching Data and Computing Resources.} $D$ is matched with $C$ according to the initial optimal strategy, ensuring secure data transfer.
    \item[\ding{205}] \textbf{Step 4: Federated Learning and Data Quality Evaluation.} $S$ and $C$ perform vanilla FL training and dynamically evaluate data quality $f_n$ of each $D_n$.
    \item[\ding{206}] \textbf{Step 5: Dynamic Adjustment of the Optimal Strategy.} $S$, $D$, and $C$ adjust the optimal strategy $\langle \eta^*, Q^* \rangle$ based on real training results.
\end{itemize}

In step 3 (line 5), we call the Gale-Shapley algorithm detailed in \textbf{Appendix \ref{App-3}}, which performs an optimal matching $G^*$ based on the calculated $d^*_m$ as preference. In Step 5, after the data quality $f_m$ is adjusted according to the training results (line 11), the model owner and the data owners sequentially adjust $\eta^*$ and $x^*_n$ at the specific training round $L$ (line 14). $L$ is set to the number of rounds that best reflects the quality of the model, depending on the dataset. Subsequently, data owners providing more data pay additional training fees to their corresponding computing centers.

% \vspace{-50pt}  % 或者根据需要调整具体数值

\section{Experiments}\label{sec_4}
In this section, we perform comprehensive experiments to answer the following research questions.

\textbf{Q1:} Does our QD-RDFL exist the optimal strategy profile to reach SNE?

\textbf{Q2:} How does the data owner's initial data quality $f_n$ affect its personal and model owner's utility?

\textbf{Q3:}  How does the incentive effectiveness of our QD-RDFL compare to random and fixed strategies on different datasets?

\textbf{Q4:} How do different components (dynamic adjustment mechanisms) affect the global model performance?

\subsection{Experiments Setting}
We conduct extensive experiments on various vision datasets, including MNIST\cite{bib-37}, CIFAR-10\cite{bib-38}, and CIFAR-100\cite{bib-33}. To simulate the variation in data quality provided by different data owners, we add Gaussian noise with varying intensity to the data from different owners. We calculate the mean square error (MSE) of the images after and before noise addition and use 1-MSE to represent the initial data quality of the data owner. The actual $f_n$ is evaluated using Equ. \ref{eq:formula7}. Furthermore, to represent the relative position of $f_n$, we normalize its value. Other detailed settings of the experiment see \textbf{Appendix \ref{App-4}}. The matching results of the data centers, the utility, and the parameter analysis are detailed in \textbf{Appendix \ref{App-6}}. Our code is available at https://anonymous.4open.science/r/QD-RDFL-4B3F.

\subsection{Experimental Results}

\textbf{SNE verification (RQ1)}. To confirm the existence and uniqueness of the SNE in QD-RDFL, we set $N=10$ on the MNIST dataset and observe the variation of utility under different strategies of the model owner and the data owner, respectively.

The uniqueness of the model owner's utility $U_s$ is demonstrated in Fig. \subref*{fig2a}. We set the quantity and the strategies of the data owners remain constant, the curve of $U_s$ is convex with increasing $\eta$, with its highest point representing the optimal response. The optimal $\eta$ is 1.33 at the Stackelberg equilibrium, and the corresponding $U_S$ is about 0.57. The uniqueness of the Nash equilibrium among the data owners is demonstrated in Fig. \subref*{fig2b}. For the selected data owner $D_n$, we retain $x_{-n}$, the initial $f_N, N=10$ and $\eta$, and then vary $x_n$. The results show that $U_n$ is convex, the maximum $U_n$ of 0.0066 is obtained when $x_n$ is 0.087. The results indicate that there are unique SNE in our method. Each data owner can maximize utility by determining the data contribution strategy in our method.

\setlength{\abovecaptionskip}{0.1cm}  
\setlength{\belowcaptionskip}{0cm}  % 设置图例下方的距离  
\setlength{\textfloatsep}{0.1cm}      % 设置图和正文之间的距离
\begin{figure}[tb]
  \centering
  \subfloat[]
  {\includegraphics[width=0.48\columnwidth]{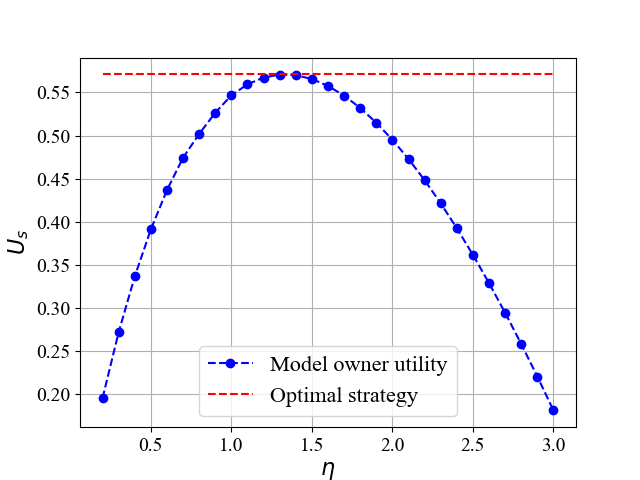}\label{fig2a}}
  \hspace{-3mm}     % 重点就在这，优先横向排列，自动换行
  \subfloat[]
  {\includegraphics[width=0.48\columnwidth]{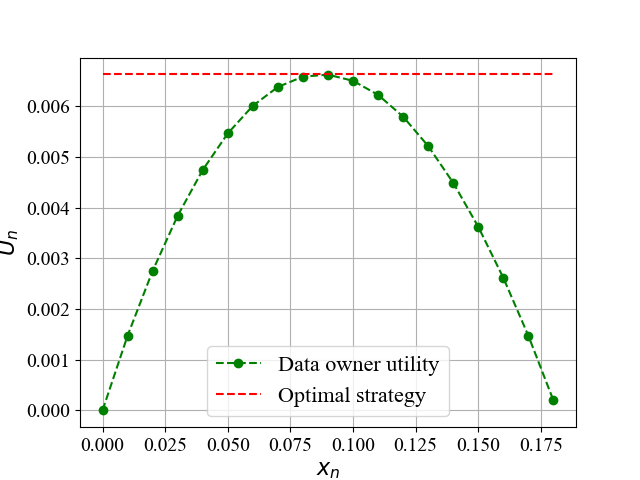}\label{fig2b}}
  \quad
  \caption{Uniqueness and existence of SNE of our method, (a) the utility of model owner with different $\eta$ and (b) the utility of data owner with different $q_n$.}
  \label{fig2}
\end{figure}

\setlength{\abovecaptionskip}{0.1cm}  
\setlength{\belowcaptionskip}{0cm}  % 设置图例下方的距离  
\setlength{\textfloatsep}{0.1cm}      % 设置图和正文之间的距离
\begin{figure}[tb]
  \centering
  \subfloat[]
  {\includegraphics[width=0.51\columnwidth]{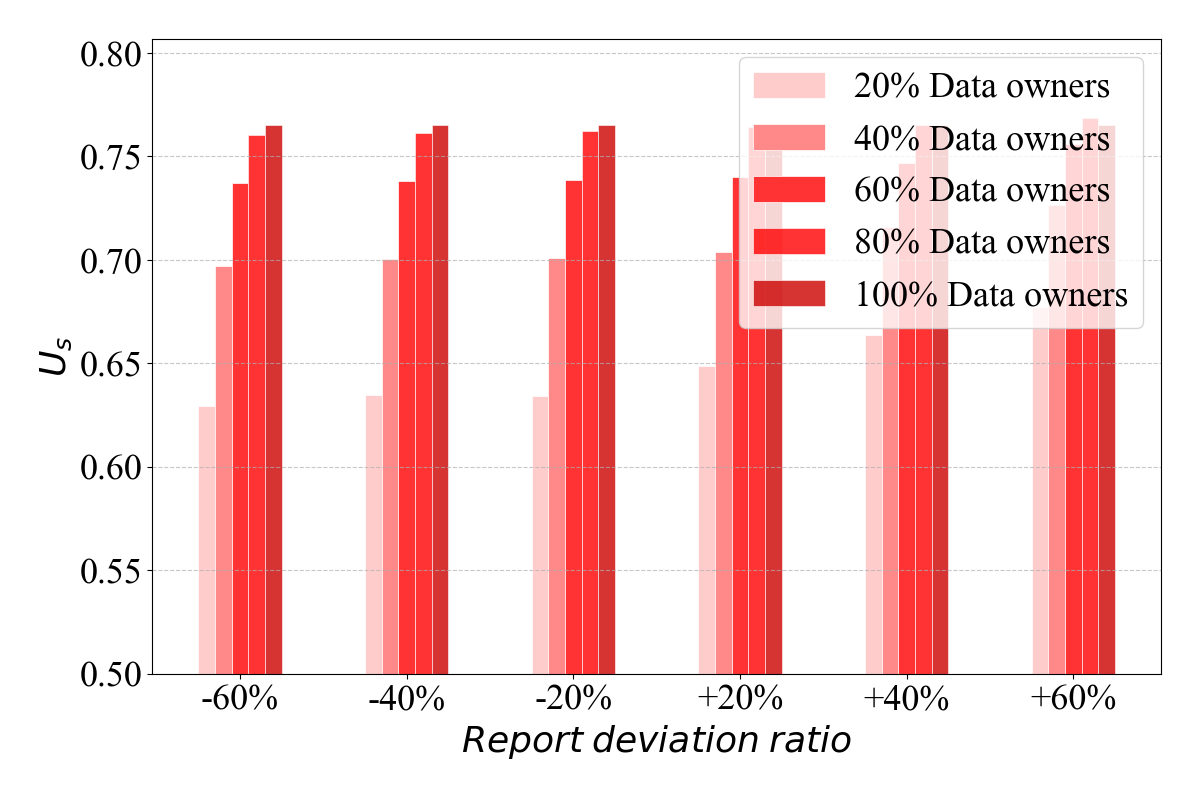}\label{fig3a}}
  \hspace{-2mm}     % 重点就在这，优先横向排列，自动换行
  \subfloat[]
  {\includegraphics[width=0.48\columnwidth]{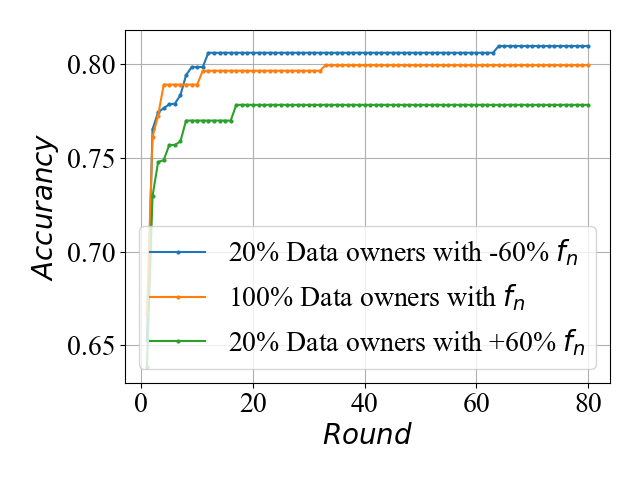}\label{fig3b}}
  \quad
  \caption{Effect of initial $f_n$ to $U_s$ and the accuracy of global model on MNIST, (a) comparison of $U_s$ with different proportions of data owners and different report deviation ratio $f_n$, and (b) the accuracy of global model with different $f_n$.}
  \label{fig3}
\end{figure}
  
\textbf{Effect of initial $f_n$ (RQ2)}. To observe the effect of the initial $f_n$ on the utility of the data owner and the utility of the model owner, we design the following two sets of experiments. 

\textbf{For data owner:} We select several data owners from $N, N=10$ for the experiment. For each chosen data owner, we fix the $f_n$ of the other data owners and increase or decrease the selected data owner’s $f_n$. We then observe the variation of $U_n$. \textbf{For model owner:} We increase or decrease all given data owner’s $f_n$ and observe the resulting variation of $U_s$.

For selected data owners, each utility $U_n$ increases as the deviation ratio of reported $f_n$ increases from -60\% to +60\% in the MNIST and CIFAR-10 datasets. For more details of the experiment, see Fig. \ref{fig-7} in \textbf{Appendix \ref{App-5}}. As shown, the results suggest that the personal utility of the data owner $U_n$ is sensitive to data quality contribution $f_n$. $U_n$ increases monotonically with the reported $f_n$. Therefore, a valid evaluation of the data quality contributions in training processes is reasonable and necessary since it prevents data owners from falsely reporting quality contributions to get additional utility.

\setlength{\abovecaptionskip}{0.1cm}  
\setlength{\belowcaptionskip}{-0.8cm}  % 设置图例下方的距离  
\setlength{\textfloatsep}{0.1cm}      % 设置图和正文之间的距离
\begin{figure*}[htb]
  \centering
  \subfloat[MNIST]
  {\includegraphics[width=0.65\columnwidth]{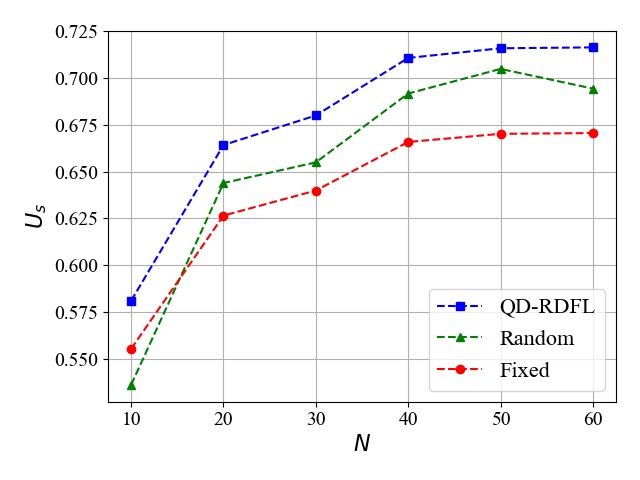}\label{fig4a}}
  \quad     % 重点就在这，优先横向排列，自动换行
  \subfloat[CIFAR-10]
  {\includegraphics[width=0.65\columnwidth]{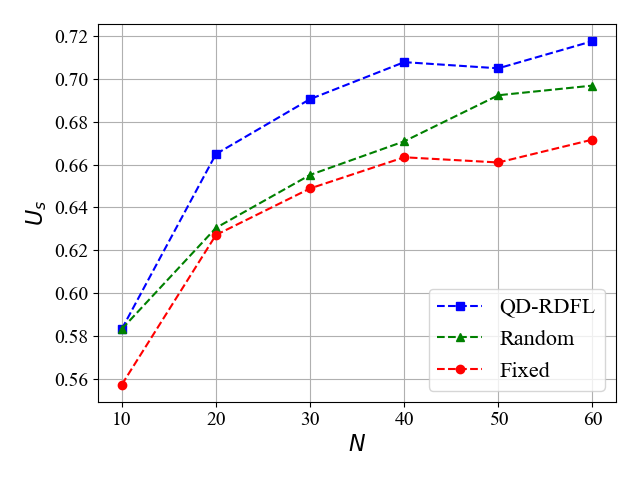}\label{fig4b}}
  \quad
  \subfloat[CIFAR-100]
  {\includegraphics[width=0.65\columnwidth]{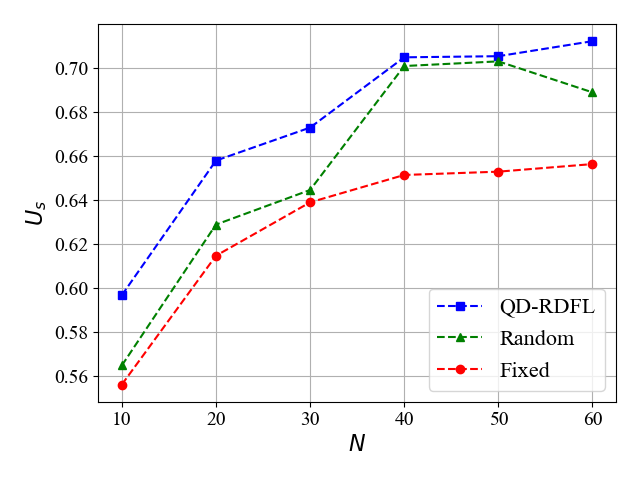}\label{fig4c}}
    \quad
  \caption{ $U_s$ with different number of data owners on different datasets.}
  \label{fig4}
\end{figure*}

\setlength{\abovecaptionskip}{0.1cm}  
\setlength{\textfloatsep}{-1.0cm}      % 设置图和正文之间的距离
\begin{figure*}[htb]
  \centering
  \subfloat[MNIST]
   {\includegraphics[width=0.65\columnwidth]{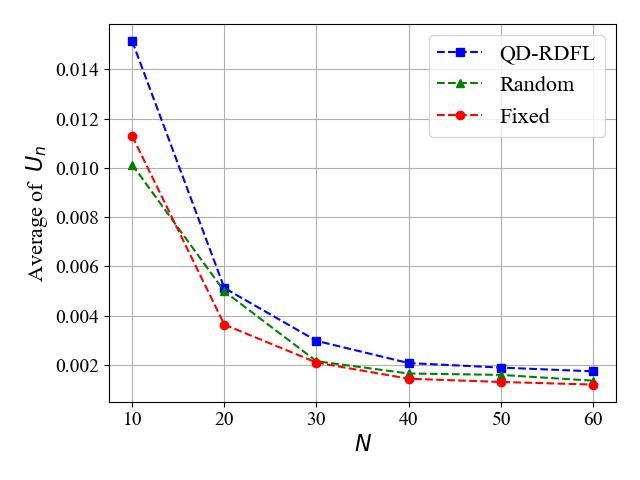}\label{fig5a}}
  \quad
  \subfloat[CIFAR-10]
  {\includegraphics[width=0.65\columnwidth]{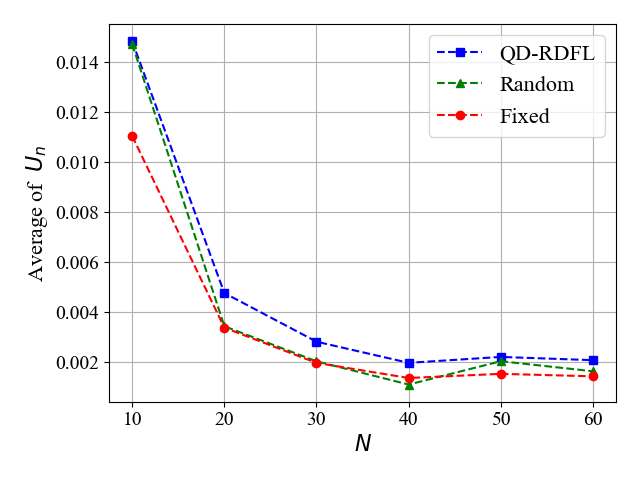}\label{fig5b}}
  \quad
  \subfloat[CIFAR-100]
  {\includegraphics[width=0.65\columnwidth]{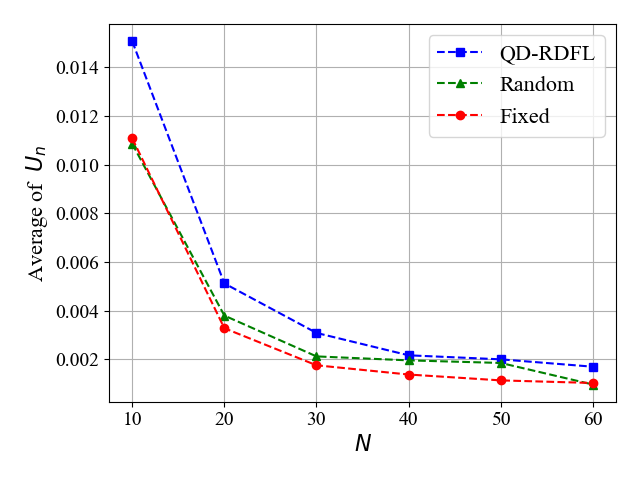}\label{fig5c}}
  \quad
  \caption{Average of $U_n$ with different number of data owners on different datasets.}
  \label{fig5}
\end{figure*}

\setlength{\abovecaptionskip}{0.1cm}  
\setlength{\belowcaptionskip}{0cm}  % 设置图例下方的距离  
\setlength{\textfloatsep}{0.1cm}      % 设置图和正文之间的距离
\begin{figure*}[htb]
  \centering
  \subfloat[MNIST]
  {\includegraphics[width=0.65\columnwidth]{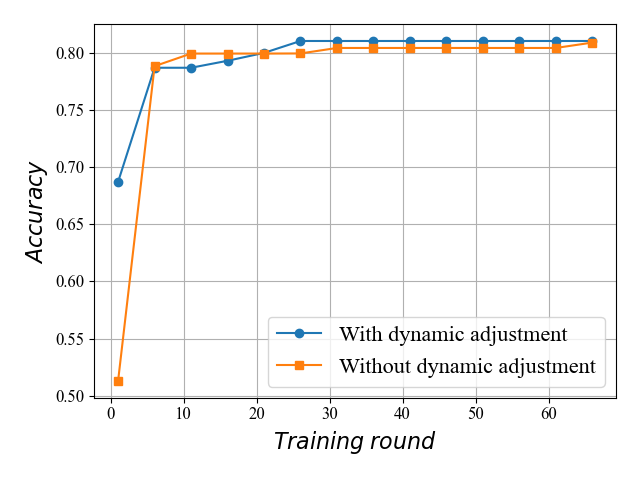}\label{fig6a}}
  \quad     % 重点就在这，优先横向排列，自动换行
  \subfloat[CIFAR-10]
  {\includegraphics[width=0.65\columnwidth]{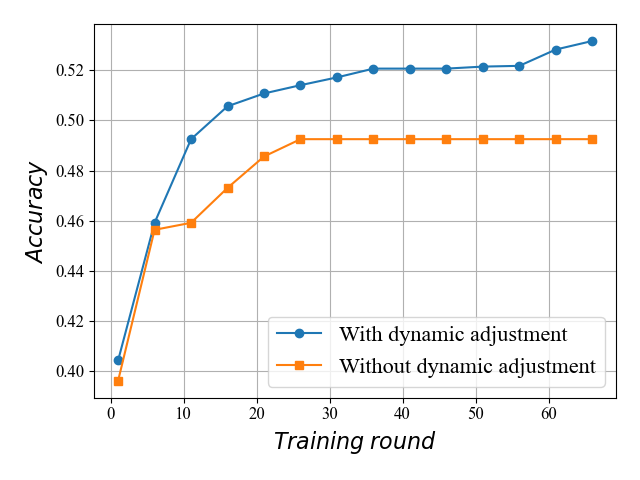}\label{fig6b}}
  \quad
  \subfloat[CIFAR-100]
  {\includegraphics[width=0.65\columnwidth]{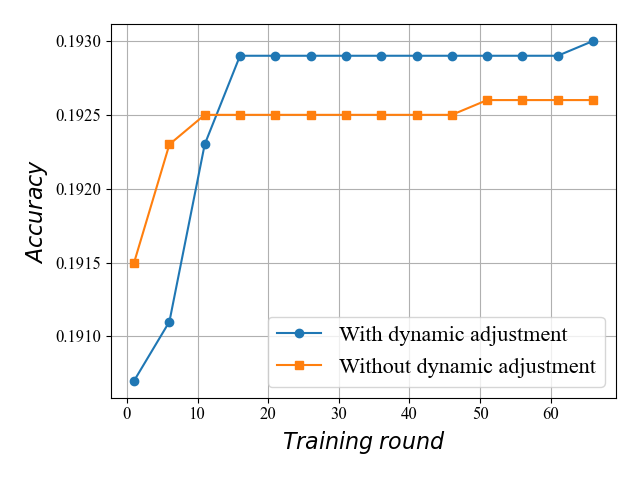}\label{fig6c}}
  \quad
  \caption{The global model performance achieved by QD-RDFL in each round on different datasets}
  \label{fig6}
\end{figure*}

For the model owner, on the MNIST dataset, we assume the varying total numbers of data owners ($20\%$ to $100\%$) with the varying deviation ratio of the reported $f_n$ ($-60\%$ to $60\%$) and observe the effect of the variation on the model owner's utility $U_s$ and global model performance (Accuracy). As shown in Fig. \subref*{fig3a}, $U_s$ increases with the increase of the report deviation ratio and the proportion of data owners. Since $f_n$ is relative, the proportion of false reports increases, resulting in the overall $f_n$ being relatively improved, which will lead to the inflated $U_s$ calculated. However, false reports lead the model owner to misjudge the model quality and obtain false utility. As the real utility is related to the global model's performance, we further explore the specific impact of report deviation on the global model's performance, in which $f_n$ of 20$\%$ data owners was down-regulated by 60$\%$ and up-regulated by 60$\%$ are set respectively. As shown in Fig.\subref*{fig3b}, the global model accuracy of the control group down-regulated by 60$\%$ is always higher than the actual situation, and the accuracy of the control group up-regulated by 60$\%$ is always lower than the actual situation. The results demonstrate that $f_n$ affects the accuracy and truly evaluates the utility of the model owner is necessary. Similar validation on CIFAR-10 is shown in Fig. \ref{fig8} in \textbf{Appendix \ref{App-5}}.

%\vspace{-5pt}

%\vspace{-5pt}

\textbf{Effect of QD-RDFL (RQ3)}. To verify the incentive effectiveness of our QD-RDFL, we perform comparison experiments with the following baselines on three different datasets.
\begin{itemize}
\item \textbf{Fixed $\eta$:} The model owner publishes an FL training task with a fixed payment value. Then, all users determine their strategies based on this payment.
\item \textbf{Random $\eta$:} The model owner publishes an FL training task with a random payment value, and all users determine their strategies based on this payment.
\item \textbf{Our method:} The model owner derives the optimal payment strategy using QD-RDFL and publishes the FL training task. Then, all users determine their strategies using this payment.
\end{itemize}

Varying numbers of data owners, we observe the utility of model owners $U_s$ and the average utility of $N$ data owners $U_n, N\in N$, respectively. As shown in Figs. \subref*{fig4a}, \subref*{fig4b} and \subref*{fig4c}, $U_s$ of QD-RDFL always make outstanding performance compared to the benchmarks compared to the benchmarks on different datasets. The results indicate that the QD-RDFL algorithm can encourage data owners to provide more quantity and higher quality data for FL, thus getting the optimal utility of the model owner. However, due to the diminishing marginal effect, the increased rate of $U_s$ gradually slows down. Figs. \subref*{fig5a}, \subref*{fig5b} and \subref*{fig5c} show the average of $U_n$ on the different datasets. As $N$ increases, the average of $U_n$ decreases. Because as $N$ increases, more data owners participate in FL training while the total model payment grows relatively slowly. As shown, the average of $U_n$ of our QD-RDFL always outperforms compared to the benchmarks. The results demonstrate that data owners can always get better utility, thus effectively encouraging them. It is worth noting that since each set of experiments generates an $\eta$ independently for the Random algorithm, the results always oscillate, but it is always lower than our method.

\textbf{Ablation study (RQ4)}. To evaluate the effectiveness of the proposed dynamic adjustment mechanism, we verify accuracy of the global model on different datasets ($N=10$) and design the following ablation experiment, in which \textbf{QD-RDFL without Dynamic Adjustment} is set. In QD-RDFL, three parties first determine the initial strategies based on the data owners reported $f_n$. Then, QD-RDFL evaluates the actual $f_n$ in the specific training round ($L=3$) and dynamically adjusts the strategies. QD-RDFL without Dynamic Adjustment means that the reported $f_n$ is always used.

In Fig. \ref{fig6}, we verify the accuracy of QD-RDFL and QD-RDFL without Dynamic Adjustment over 70 training iterations on different datasets. The results show that the accuracy of QD-RDFL is better than that of QD-RDFL without Dynamic Adjustment. It illustrates that the dynamic adjustment mechanism performs a more accurate evaluation of data quality, thereby successfully encouraging data owners to provide high-quality and large-quantity data and improving the performance of the global model.

\section{Conclusions and Future works}\label{sec_5}
In this work, we explore a method for data assetization through FL.
We first propose the resource-decoupled FL framework to address the challenges of data resource heterogeneity and the mismatch between data and computing resources. Building on this, to address the potential issue of participant willingness to engage, we formulate a Tripartite Stackelberg model. Our theoretical analysis guarantees the existence of a Stackelberg-Nash equilibrium that optimizes global utility and achieves the goal of maximizing the value of data assets. We use the contribution of data to the model as a standard for assessing data quality, incentivizing data owners to enhance their contributions to model quality in order to achieve better personal gains. Our comprehensive experiments show that our method can achieve optimal global utility in various scenarios. In future work, we aim to expand the experimental scope to validate the effectiveness of our algorithm across more datasets and further optimize the global model by incorporating a selection mechanism for the model owner.

%%
%% The next two lines define the bibliography style to be used, and
%% the bibliography file.
\bibliographystyle{ACM-Reference-Format}
\bibliography{sample-base}

%%
%% If your work has an appendix, this is the place to put it.
\appendix

\section{Appendix}

\subsection{Proof of Lemma 2}\label{App-1}
In order to give the optimal Nash equilibrium solution for $N$ data owners, we give the following proof.
According to Equ.\ref{equ-8}, we set this first derivative equal to 0 for each data owner, then we get:
\begin{equation}\label{eq:formula 16}
    \frac{\eta}{\sum^N_{i=1}q_i}-\frac{\eta q_n}{(\sum^N_{i=1}q_i)^2}=\frac{\lambda\rho}{f_n}
\end{equation}
And then we consider $N$ data owners:

\begin{equation}\label{eq:formula 17}
\left\{
\begin{aligned}
%\nonumber
\frac{\eta}{\sum^N_{i=1}q_i}&-\frac{\eta q_1}{(\sum^N_{i=1}q_i)^2}=\frac{\lambda\rho}{f_1}\\
\frac{\eta}{\sum^N_{i=1}q_i}&-\frac{\eta q_2}{(\sum^N_{i=1}q_i)^2}=\frac{\lambda\rho}{f_2}\\
&\vdots\\
\frac{\eta}{\sum^N_{i=1}q_i}&-\frac{\eta q_N}{(\sum^N_{i=1}q_i)^2}=\frac{\lambda\rho}{f_N}\\
\end{aligned}
\right.
\end{equation}
We sum up the Equ.\ref{eq:formula 17}:
\begin{equation}\label{eq:formula 18}
(N-1)\eta=\sum^N_{i=1}\frac{\lambda\rho}{f_i}\sum^N_{i=1}q_i
\end{equation}
So we get:
\begin{equation}\label{eq:formula 19}
\sum^N_{i=1}q_i=\frac{(N-1)\eta}{\sum^N_{i=1}\frac{\lambda\rho}{f_i}}
\end{equation}

We plug Equ.\ref{eq:formula 19} into Equ.\ref{eq:formula 16} and get:
\begin{equation}\label{eq:formula 20}
\sum_{i \neq n}q_i=\frac{(N-1)^2\eta^2}{f_n\lambda\rho(\sum^N_{i=1}\frac{1}{f_i})^2}
\end{equation}

Thus, the utility function is quasi-concave, so a unique Nash equilibrium exists.

\subsection{Proof of Theorem 1}\label{App-2}

We can follow the proof of Theorem 1.

The existence and uniqueness of SNE can be deduced by the property that the strategy space is a convex and compact subspace of the Euclidean space while the utility functions are concave. In the case of the model owner, it can be justified that the optimal utility $U_s$ is obtained only at $\eta^*$ derived by direct derivation due to the strictly concave property of Equ. \ref{equ-4}, leading to the first inequality in Def. \ref{def-2} holding uniquely at $\eta^*$, according to Lemma \ref{lemma-3}. Similarly, a unique Nash equilibrium can be justified, and the second inequality holds for every data owner at $q_n^*$, according to Lemma \ref{lemma-2}. However, for privacy reasons, we don't split the data in the computing center; instead, we use the Gale-Shapley algorithm to achieve one-time matching with the data owner according to the optimal $d_m^*$, according to Lemma \ref{lemma-1}. According to \cite{bib-30,bib-31}, the Gale-Shapley algorithm satisfies the relaxed Nash equilibrium, thus satisfying the third inequality. Therefore, SNE exists since the set of inequalities in Equ. \ref{equ-6} can be satisfied at $\left\langle \eta^*, Q^*, G^* \right\rangle$.

\subsection{Gale-Shapley Algorithm}\label{App-3}
In this work, we adopt the Gale-Shapley algorithm to find the match between the computing centers and the data owners. The details and pseudocode are shown in Alg.3.

\begin{algorithm}[h]
\caption{Gale-Shapley}
\begin{algorithmic}[1]
\REQUIRE $D=\{D_1, D_2, ..., D_N\}$, $C=\{C_1, C_2, ... , C_M\}$, and $\sigma_m$
\ENSURE $G=\{G_1, G_2, ... , G_N\}$
\STATE Gets $C$ sorted in descending order by $\sigma_m$ represented as $C^*$
\FOR{each $D_n, n \in [1, N]$}
    \STATE $proposal=C^*$
\ENDFOR

\FOR{each $C_m, m \in [1, M]$}
    \STATE Calculate $d_m^*$ according to Equ.\ref{equ-31}
    \FOR{each $D_n, n \in [1, N]$}
        \STATE $preference=|d_m^*-x_n^*|$
    \ENDFOR
    \STATE sort $preference$ in ascending order
\ENDFOR
\WHILE{$\exists D_n, n \in [1, N]$ $s.t.$ $G_n=\varnothing$ }
    \FOR{$D_n, n \in [1, N]$ $s.t.$ $G_n=\varnothing$}
        \STATE $c\leftarrow$ $\max(proposal)$ to whom $D_n$ has not proposed yet
        \IF{$\exists D_j, j \in [1, N]$ $s.t.$ $G_j=c$}
            \IF{ $preference[c]$ of $D_n$ is greater $D_j$}
                \STATE $G_n \leftarrow c$, $G_j \leftarrow \varnothing$
            \ENDIF
        \ELSE
            \STATE $G_n \leftarrow c$ 
        \ENDIF
    \ENDFOR
\ENDWHILE
\RETURN $G$
\end{algorithmic}
\end{algorithm}

\subsection{Experiments Setting}\label{App-4}

Our training utilizes a widely used and well-established model architecture: a Convolutional Neural Network (CNN). For the simpler MNIST dataset, we use CNN, which consists of two convolutional layers and two fully connected layers. For datasets CIFAR-10 and CIFAR-100, we use CNN composed of three convolutional layers and four fully connected layers.
The FL training process involves 100 iterations, including three local epochs of client training. Cross-entropy loss is employed for all classification tasks. Meanwhile, we adopt the Adam optimizer with a learning rate of 0.001. All experiments are conducted on PyTorch  1.11.0 and RTX 4090. In the absence of special instructions, the specific parameters of the experiment in this paper are set in Table \ref{tab3}.

\begin{table}[htb]
\centering
\caption{Experiment Setup}
\begin{tabular}{l l}
\toprule
\textbf{Parameters} & \makecell[c]{\textbf{Values}}\\ \midrule
$\lambda$ (Market regulating factor)  & \makecell[c]{1} \\
$\rho$ (Unit training payment of $D$)  & \makecell[c]{1} \\
$f_n$ (Data quality of training by $D_n$)      & \makecell[c]{$U\sim[0,1]$} \\
$\varepsilon$ (Unit training cost of $C$)      & \makecell[c]{1} \\
$\sigma_m$ (Computational power of $C_m$)         & \makecell[c]{$U\sim[0,1]$} \\
$\alpha$ (Adjustment factor of model quality)     & \makecell[c]{5} \\
$f_m$ (Data quality of $C_m$ calculate during training) & \makecell[c]{$U\sim[0,1]$} \\
$\beta$ (Learning rate for local model parameter update)  & \makecell[c]{$1e-5$} \\
\bottomrule
\end{tabular}

\label{tab3}
\end{table}

\subsection{Supplementary Results of (RQ2)}\label{App-5}

In MNIST and CIFAR-10, we randomly selected three data owners from $N, N=10$ for the experiment. For each selected data owner, we fix the values of the other data owners, varying their $f_n$ from $-60\%$ to $+60\%$ of the actual value of $f_n$, and observe the variation of $U_n$. The results in Figs. \subref*{fig7a} and \subref*{fig7b} show that, as the deviation ratio of data owner reported $f_n$ increases, $U_n$ increases on different datasets. However, false data quality increases the personal utility value, affecting other high-quality data owners' utility in making a proper decision, thus affecting the final global utility.

To further observe the effects of initial $f_n$ on $U_s$, we also performed the same experiments on the CIFAR-10 dataset. We assume the deviation ratio of the reported $f_n$ from $-60\%$ to $60\%$ with the total number of data owners on CIFAR-10 from $20\%$ to $100\%$, and observe the effect of $U_s$ and the performance of the global model. As shown in Figure \subref*{fig8a}, $U_s$ increases with the reported deviation rate and the percentage of data owners. Since $f_n$ is relative, it increases as the reported deviation rate increases, increasing $U_s$. In the same way, we can conclude that the wrong report will cause the model owner to get the wrong utility value. Then, we also discuss the specific effects of report deviation of $f_n$ in Fig. \subref*{fig8b}. The model accuracy of the control group down-regulated by 60$\%$ is always higher than the real situation, and the model accuracy of the control group up-regulated by 60$\%$ is always lower than the real situation. The results demonstrate that $f_n$ affects the accuracy of the global model, and the evaluation of $f_n$ is necessary.

\begin{figure}[tb]
  \centering
  \subfloat[]
  {\includegraphics[width=0.50\columnwidth]{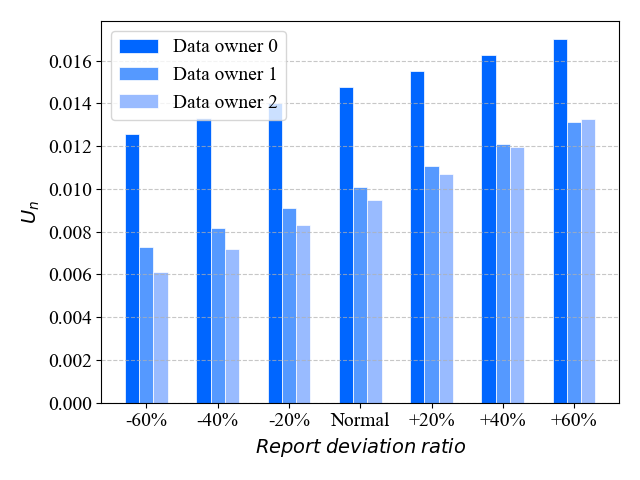}\label{fig7a}}
  \hspace{-2mm}     % 重点就在这，优先横向排列，自动换行
  \subfloat[]
  {\includegraphics[width=0.50\columnwidth]{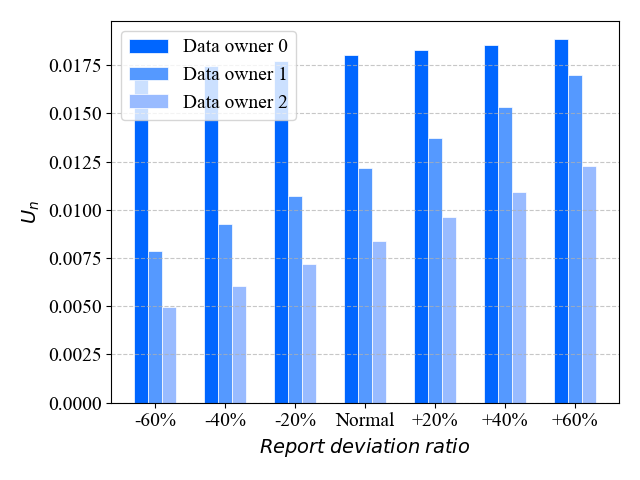}\label{fig7b}}
  \quad
  \caption{Effect of the initial $f_n$ on the utility of the data owner on different datasets, (a) MNIST and (b) CIFAR-10.}
  \label{fig-7}
\end{figure}
%\vspace{-10pt}

\begin{figure}[tb]
  \centering
  \subfloat[]
  {\includegraphics[width=0.51\columnwidth]{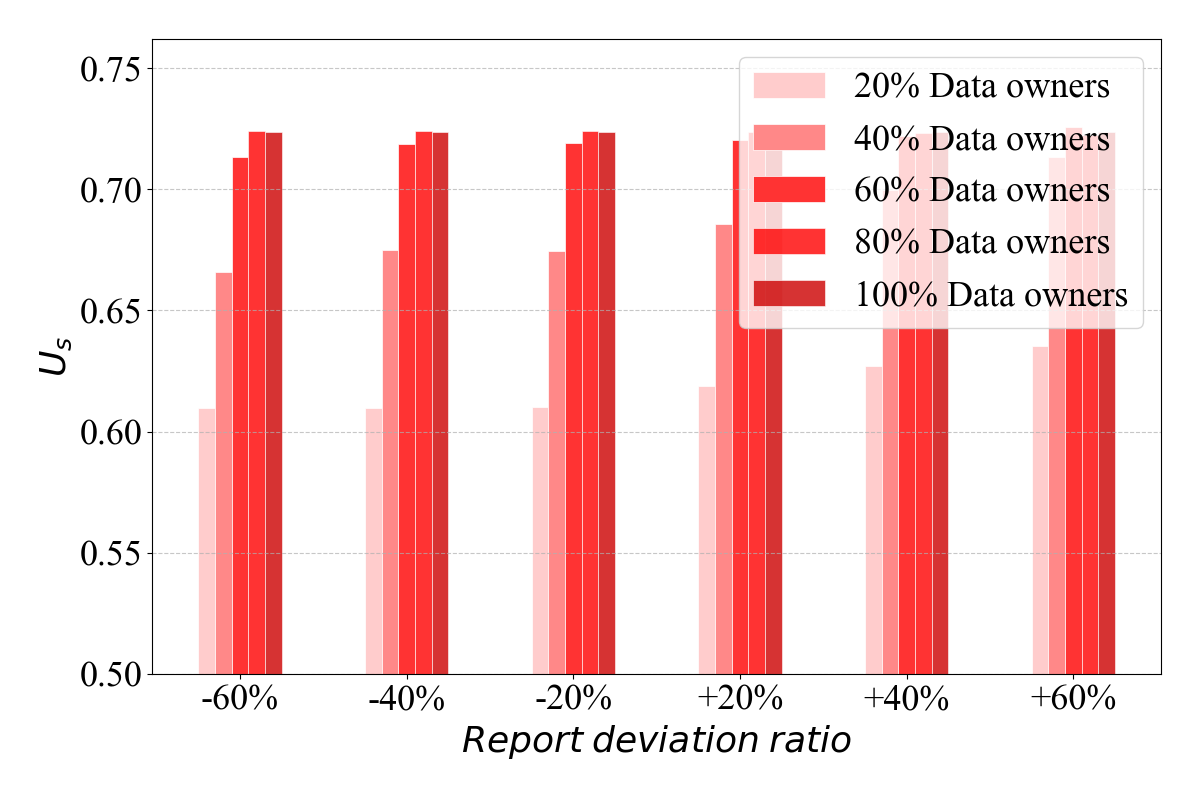}\label{fig8a}}
  \hspace{-2mm}     % 重点就在这，优先横向排列，自动换行
  \subfloat[]
  {\includegraphics[width=0.48\columnwidth]{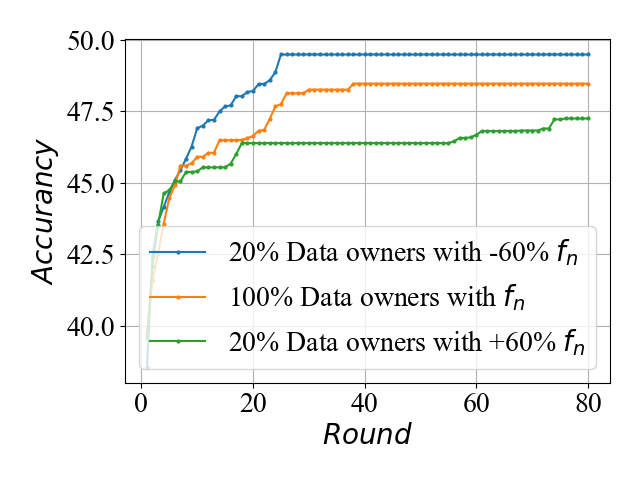}\label{fig8b}}
  \quad
  \caption{Effect of reported $f_n$ to $U_s$ and the accuracy of global model on CIFAR-10, (a) $U_s$ with different proportions of data owners and different report deviation ratio $f_n$, and (b) the accuracy of global model with different $f_n$.}
  \label{fig8}
\end{figure}
%\vspace{-10pt}

\subsection{Experiment on Computing Center}\label{App-6}

To observe the effect of the matching and parameter analysis of computing centers, we set $M=10$ and $N=10$ and observe the matching results on the different datasets. As shown in Tables \ref{tab4}, \ref{tab5} and \ref{tab6}, for each computing center, they first calculate the $d^*_m$ according to Equ.\ref{equ-32}. Then, the value of $preference$ makes the highest priority to the data owner who provides $x^*_n$ closer to $d^*_m$. At the same time, every data owner tends to choose the computing center with high $\sigma_m$. Here, we randomly set $\sigma_m \in [0,1]$ of the computing centers with o.1 step size. Experimental results show that our algorithm can obtain the approximate optimal amount of data based on $d^*_m$ and $x^*_n$. In fact, without considering the privacy issue, we can make the data center undertakes exactly $d^*_m$ of the data to optimal its utility, but this violates the privacy constraint. It should be noted that since the data owner will choose inversely according to the value of $\sigma_m$, the matching result is the optimal solution that satisfies the wishes of both parties. For example, in Table 4, $d_m$ of $C_1$ is 1327, and $x_n$ of $D_10$ is also 1327, but they do not match since $\sigma_1$ is only 0.1.

\begin{table}[htb]
\centering
\caption{Matching results, $U_m$, and related parameters on MNIST}
\setlength{\tabcolsep}{0.7mm} % 增加列间距
\begin{tabular}{l l l l l l l l l l l}
\toprule
\makecell[c]{\textbf{}} & \makecell[c]{\textbf{$D_1$}} & \makecell[c]{\textbf{$D_2$}} & \makecell[c]{\textbf{$D_3$}} & \makecell[c]{\textbf{$D_4$}} & \makecell[c]{\textbf{$D_5$}} & \makecell[c]{\textbf{$D_6$}} & \makecell[c]{\textbf{$D_7$}} & \makecell[c]{\textbf{$D_8$}} & \makecell[c]{\textbf{$D_9$}} & \makecell[c]{\textbf{$D_{10}$}} \\ \midrule
\makecell[c]{\textbf{$C$}} & \makecell[c]{$C_5$} & \makecell[c]{$C_1$} & \makecell[c]{$C_2$} & \makecell[c]{$C_9$} & \makecell[c]{$C_8$} & \makecell[c]{$C_6$} & \makecell[c]{$C_4$} & \makecell[c]{$C_7$} & \makecell[c]{$C_{10}$} & \makecell[c]{$C_3$} \\
\makecell[c]{\textbf{$\sigma_m$}} & \makecell[c]{0.5} & \makecell[c]{0.1} & \makecell[c]{0.2} & \makecell[c]{0.9} & \makecell[c]{0.8} & \makecell[c]{0.6} & \makecell[c]{0.4} & \makecell[c]{0.7} & \makecell[c]{1.0} & \makecell[c]{0.3} \\
\makecell[c]{\textbf{$d_m$}} & \makecell[c]{798} & \makecell[c]{1327} & \makecell[c]{1195} & \makecell[c]{269} & \makecell[c]{401} & \makecell[c]{666} & \makecell[c]{930} & \makecell[c]{533} & \makecell[c]{136} & \makecell[c]{1062} \\
\makecell[c]{\textbf{$x_n$}} & \makecell[c]{584} & \makecell[c]{446} & \makecell[c]{577} & \makecell[c]{157} & \makecell[c]{445} & \makecell[c]{719} & \makecell[c]{1120} & \makecell[c]{556} & \makecell[c]{136} & \makecell[c]{1327} \\
\makecell[c]{\textbf{$U_m$}} & \makecell[c]{0.069} & \makecell[c]{0.127} & \makecell[c]{0.110} & \makecell[c]{0.015} & \makecell[c]{0.008} & \makecell[c]{0.021} & \makecell[c]{0.091} & \makecell[c]{0.040} & \makecell[c]{0.000} & \makecell[c]{0.088} \\
\bottomrule
\end{tabular}
\label{tab4}
\end{table}

\begin{table}[htb]
\centering
\caption{Matching results, $U_m$, and related parameters on CIFAR-10}
\setlength{\tabcolsep}{0.7mm} % 增加列间距
\begin{tabular}{l l l l l l l l l l l}
\toprule
\makecell[c]{\textbf{}} & \makecell[c]{\textbf{$D_1$}} & \makecell[c]{\textbf{$D_2$}} & \makecell[c]{\textbf{$D_3$}} & \makecell[c]{\textbf{$D_4$}} & \makecell[c]{\textbf{$D_5$}} & \makecell[c]{\textbf{$D_6$}} & \makecell[c]{\textbf{$D_7$}} & \makecell[c]{\textbf{$D_8$}} & \makecell[c]{\textbf{$D_9$}} & \makecell[c]{\textbf{$D_{10}$}} \\ \midrule
\makecell[c]{\textbf{$C$}} & \makecell[c]{$C_5$} & \makecell[c]{$C_7$} & \makecell[c]{$C_6$} & \makecell[c]{$C_9$} & \makecell[c]{$C_8$} & \makecell[c]{$C_4$} & \makecell[c]{$C_{10}$} & \makecell[c]{$C_1$} & \makecell[c]{$C_3$} & \makecell[c]{$C_2$} \\
\makecell[c]{\textbf{$\sigma_m$}} & \makecell[c]{0.5} & \makecell[c]{0.7} & \makecell[c]{0.6} & \makecell[c]{0.9} & \makecell[c]{0.8} & \makecell[c]{0.4} & \makecell[c]{1.0} & \makecell[c]{0.1} & \makecell[c]{0.3} & \makecell[c]{0.2} \\
\makecell[c]{\textbf{$d_m$}} & \makecell[c]{1183} & \makecell[c]{756} & \makecell[c]{969} & \makecell[c]{329} & \makecell[c]{542} & \makecell[c]{1396} & \makecell[c]{115} & \makecell[c]{2037} & \makecell[c]{1610} & \makecell[c]{1823} \\
\makecell[c]{\textbf{$x_n$}} & \makecell[c]{1795} & \makecell[c]{820} & \makecell[c]{199} & \makecell[c]{408} & \makecell[c]{715} & \makecell[c]{2036} & \makecell[c]{115} & \makecell[c]{148} & \makecell[c]{183} & \makecell[c]{179} \\
\makecell[c]{\textbf{$U_m$}} & \makecell[c]{0.078} & \makecell[c]{0.040} & \makecell[c]{0.052} & \makecell[c]{0.013} & \makecell[c]{0.023} & \makecell[c]{0.093} & \makecell[c]{0.000} & \makecell[c]{0.043} & \makecell[c]{0.042} & \makecell[c]{0.093} \\
\bottomrule
\end{tabular}
\label{tab5}
\end{table}

\begin{table}[htb]
\centering
\caption{Matching results, $U_m$, and related parameters on CIFAR-100}
\setlength{\tabcolsep}{0.7mm} % 增加列间距
\begin{tabular}{l l l l l l l l l l l}
\toprule
\makecell[c]{\textbf{}} & \makecell[c]{\textbf{$D_1$}} & \makecell[c]{\textbf{$D_2$}} & \makecell[c]{\textbf{$D_3$}} & \makecell[c]{\textbf{$D_4$}} & \makecell[c]{\textbf{$D_5$}} & \makecell[c]{\textbf{$D_6$}} & \makecell[c]{\textbf{$D_7$}} & \makecell[c]{\textbf{$D_8$}} & \makecell[c]{\textbf{$D_9$}} & \makecell[c]{\textbf{$D_{10}$}} \\ \midrule
\makecell[c]{\textbf{$C$}} & \makecell[c]{$C_7$} & \makecell[c]{$C_5$} & \makecell[c]{$C_8$} & \makecell[c]{$C_2$} & \makecell[c]{$C_6$} & \makecell[c]{$C_9$} & \makecell[c]{$C_{10}$} & \makecell[c]{$C_3$} & \makecell[c]{$C_1$} & \makecell[c]{$C_4$} \\
\makecell[c]{\textbf{$\sigma_m$}} & \makecell[c]{0.7} & \makecell[c]{0.5} & \makecell[c]{0.8} & \makecell[c]{0.2} & \makecell[c]{0.6} & \makecell[c]{0.9} & \makecell[c]{1.0} & \makecell[c]{0.3} & \makecell[c]{0.1} & \makecell[c]{0.4} \\
\makecell[c]{\textbf{$d_m$}} & \makecell[c]{726} & \makecell[c]{1152} & \makecell[c]{513} & \makecell[c]{1791} & \makecell[c]{939} & \makecell[c]{300} & \makecell[c]{87} & \makecell[c]{1578} & \makecell[c]{2004} & \makecell[c]{1365} \\
\makecell[c]{\textbf{$x_n$}} & \makecell[c]{714} & \makecell[c]{1807} & \makecell[c]{452} & \makecell[c]{110} & \makecell[c]{440} & \makecell[c]{273} & \makecell[c]{86} & \makecell[c]{148} & \makecell[c]{98} & \makecell[c]{2003} \\
\makecell[c]{\textbf{$U_m$}} & \makecell[c]{0.026} & \makecell[c]{0.081} & \makecell[c]{0.033} & \makecell[c]{0.048} & \makecell[c]{0.032} & \makecell[c]{0.015} & \makecell[c]{0.000} & \makecell[c]{0.112} & \makecell[c]{0.096} & \makecell[c]{0.076} \\
\bottomrule
\end{tabular}
\label{tab6}
\end{table}

\end{document}